\newcommand{\CLASSINPUTinnersidemargin}{0.75in} 
\newcommand{\CLASSINPUToutersidemargin}{0.75in} 
\newcommand{\CLASSINPUTtoptextmargin}{0.75in} 
\newcommand{\CLASSINPUTbottomtextmargin}{0.75in}

\documentclass[conference, letterpaper, 10pt]{IEEEtran}

\IEEEoverridecommandlockouts    

\usepackage[normalem]{ulem}	                        
\usepackage[table,usenames,dvipsnames]{xcolor}      
\usepackage{extarrows}                              
\usepackage{paralist}
\usepackage{acronym}
 \usepackage[numbers]{natbib}
\usepackage{wasysym}

\usepackage{amsmath,amssymb,amsfonts,amsthm,dsfont} 
\usepackage{algorithm,algorithmicx,listings}        
\usepackage[noend]{algpseudocode}			              

\usepackage{graphicx}
\usepackage{tabularx}
\usepackage[font={small}]{caption}   
\usepackage[font={small}]{subcaption}
\usepackage{tikz}
\usetikzlibrary{shapes,snakes}
\usetikzlibrary{arrows,fit,shapes,automata}
\usetikzlibrary{positioning,fit,calc,shapes}
\usetikzlibrary{decorations.fractals}
\usetikzlibrary{decorations.markings}

\usepackage{hyperref}
\newcommand{\truev}{\mathsf{true}}

\providecommand{\abs}[1]{\lvert#1\rvert}

\newcommand{\sink}{\mathsf{sink}}

\newcommand{\calM}{\mathcal{M}}

\newcommand{\calA}{\mathcal{A}}

\newcommand{\calG}{\mathcal{G}}
\newcommand{\calAP}{\mathcal{AP}}

\definecolor{mygrey}{HTML}{737373}

\newcommand{\sq}{\mathsf{Sq}}
\newcommand{\cir}{\mathsf{Cir}}
\newcommand{\hex}{\mathsf{Hex}}
\newcommand{\dia}{\mathsf{Dia}}
\newcommand{\tri}{\mathsf{Tri}}

\newcommand{\level}{\mathsf{Level}}

\acrodef{gssp}[GSSP]{generalized stochastic shortest path}
\acrodef{mcts}[MCTS]{Monte-Carlo Tree Search}
\acrodef{ltl}[LTL]{linear temporal logic}
\acrodef{dfa}[DFA]{deterministic finite-state automaton}
\acrodef{mdp}[MDP]{Markov decision process}
\acrodef{pomdps}[POMDPs]{partially observable Markov decision processes}

\acrodef{slam}[SLAM]{simultaneous localization and mapping}

\newtheorem{theorem}{Theorem}
\newtheorem{proposition}{Proposition}

\newtheorem{definition}{Definition}

\newtheorem*{assumption*}{Assumption}

\newtheorem*{problem*}{Problem}
\newtheorem{problem}{Problem}

\newtheorem{example}{Example}
\newtheorem*{summary*}{Summary}

\def\thetitle{Optimal Temporal Logic Planning in Probabilistic Semantic Maps}
\def\theauthor{Jie Fu, Nikolay Atanasov, Ufuk Topcu, and George J. Pappas}

\begin{document}
\title{\vspace{0.3in}\LARGE \bf \thetitle}

\author{\theauthor
  \thanks{J. Fu, N. Atanasov, and G. Pappas are with the Department of
    Electrical and Systems Engineering, University of Pennsylvania,
    Philadelphia, PA 19104, USA, {\tt\small\{jief, atanasov,
      pappasg\}@seas.upenn.edu}.}%
  \thanks{U. Topcu is with the Department of Aerospace Engineering and
    Engineering Mechanics, University of Texas, Austin, TX 78712, USA,
    {\tt\small utopcu@utexas.edu}.}%
  \thanks{ This work is supported by AFRL \# FA8650-15-C-2546, ONR \#
    N000141310778, NSF \# 1550212, and TerraSwarm, one of six centers
    of STARnet, a Semiconductor Research Corporation program sponsored
    by MARCO and DARPA.} }
\maketitle

\begin{abstract}
  This paper considers robot motion planning under temporal logic
  constraints in probabilistic maps obtained by semantic
  simultaneous localization and mapping (SLAM). The uncertainty in a
  map distribution presents a great challenge for obtaining
  correctness guarantees with respect to the linear temporal logic
  (LTL) specification. We show that the problem can be formulated as an
  optimal control problem in which both the semantic map
  and the logic formula evaluation are stochastic. Our first
  contribution is to reduce the stochastic control problem for a
  subclass of LTL to a deterministic shortest path problem by
  introducing a confidence parameter $\delta$. A robot trajectory
  obtained from the deterministic problem is guaranteed to have 
  minimum cost and to satisfy the logic specification in the true
  environment with probability $\delta$. Our second contribution is to
  design an admissible heuristic function that guides the planning in
  the deterministic problem towards satisfying the
  temporal logic specification. This allows us to obtain an optimal
  and very efficient solution using the A* algorithm. The performance
  and correctness of our approach are demonstrated in a simulated
  semantic environment using a differential-drive robot.

\end{abstract}

\section{Introduction}
\label{sec:introduction}

This paper addresses robot motion planning in uncertain environments with tasks specified by \ac{ltl} co-safe formulas. A map distribution, obtained from a semantic \ac{slam} algorithm \citep{Nuchter_RAS08,Atanasov_SemanticLocalization_RSS14,slam++,semantic_sfm}, facilitates natural robot task specifications in terms of objects and landmarks in the environment. For example, we can require a robot to ``go to a room where there is a desk and two chairs'' instead of giving it exact target coordinates.  One could even describe tasks when the entire map is not available but is to be obtained as the robot explores its environment. Meanwhile, temporal logic allows one to specify rich, high-level robotic tasks. Hence, a meaningful question we aim to answer is the following. \textit{Given a semantic map distribution, how does one design a control policy that enables the robot to efficiently accomplish temporal logic tasks with high probability, despite the uncertainty in the true environment?}



This question is motivated by two distinct lines of work, namely, control under temporal logic constraints and multi-task \ac{slam}. Control synthesis with temporal logic specifications has been studied for both deterministic \citep{kloetzer2008fully,rrg_mu_calc,bhatia2010sampling} and stochastic systems \citep{lahijanian2010motion,ding2011mdp}.  A recent line of work focuses on the design problem in the presence of unknown and uncertain environments. In general, three types of uncertainty are considered: sensor uncertainty~\citep{Johnson01052015}, incomplete environment models~\citep{kress2009temporal,guo2013revising,livingston2012backtracking, livingston2013patching}, or uncertainty in the robot dynamics~\citep{wolff2012robust,fj_rss2014}. \citet{Johnson01052015} employ a model checking algorithm to evaluate the fragility of the control design with respect to temporal logic tasks when sensing is uncertain. To handle unexpected changes in the environment and incompleteness in the environment model, \citet{kress2009temporal} develop a sensor-based reactive motion planning method that guarantees the correctness of the robot behaviors under temporal logic constraints. \citet{livingston2012backtracking,livingston2013patching} propose a way to efficiently modify a nominal controller through local patches for assume-guarantee \ac{ltl} formulas.  \citet{guo2013revising} develop a revision method for online planning in a gradually discovered environment. Probabilistic uncertainty is studied in \citep{wolff2012robust,fj_rss2014}. \citet{wolff2012robust} develop a robust control method with respect to temporal logic constraints in a stochastic environment modeled as an interval Markov decision process (MDP). \citet{fj_rss2014} develop a method that learns a near-optimal policy for temporal logic constraints in an initially unknown stochastic environment. However, existing work abstracts the system and its environment into discrete models, such as, MDPs, two-player games, and plans in the discrete state space. In this work, the environment uncertainty is represented by a continuous map distribution which makes methods for discrete systems not applicable.  Moreover, control design with map distributions obtained by uncertain sensor and semantic \ac{slam} has not been addressed in literature.



While temporal logic is expressive in specifying a wide range of robot behaviors, recent advances in \ac{slam} motivate the integration of task planning with simultaneously discovering an initially unknown environment using \ac{slam} algorithms. Multi-task \ac{slam} is proposed in \citet{guez2010multi}.  The authors consider a planning problem in which a mobile robot needs to map an unknown environment, while localizing itself and maximizing long-term rewards. The authors formulate the decision-making problem as a partially observable Markov decision process and plan with both the mean of the robot pose and the mean of the map distribution.  \citet{bachrach2012estimation} develop a system for visual odometry and mapping using an RGB-D camera. The authors employ the Belief Roadmap algorithm \citep{missiuro2006adapting} to generate the shortest path from the mean robot pose to a goal state, while propagating uncertainties along the path. It is difficult, however, to extend these approaches to temporal logic planning with probabilistic semantic maps. Unlike reachability and reward maximization, the performance criteria induced by \ac{ltl} formulas require a rigorous way to reason about the uncertainty in the map distribution.  To tackle these challenges, our method brings together the notions of robustness and probabilistic correctness to satisfy quantitative temporal logic specifications in the presence of environment uncertainty given by map distributions.

 This work makes the following contributions: \begin{itemize}
  \item We formulate a stochastic optimal control problem for planning robot motion in a probabilistic semantic map under temporal logic constraints.
  \item For a subclass of LTL we reduce the stochastic problem to a deterministic shortest path problem that can be solved very efficiently. We prove that for a given confidence parameter $\delta$, the robot trajectory obtained from the deterministic problem, if it exists, satisfies the logic specification with probability $\delta$ in the true environment.
  \item We design an admissible heuristic for A* to compute the optimal solution of the deterministic problem efficiently.
\end{itemize}


\section{Problem Formulation}
\label{sec:problem} 
In this section, we introduce models for the robot and its uncertain environment, represented by a semantic map distribution. Using temporal logic as the task specification language, we formulate a stochastic optimal control problem. 

\subsection{Robot and environment models}
Consider a mobile robot whose dynamics are governed by the following discrete-time motion model:
\begin{equation}
\label{eq:system}
x_{t+1}= f(x_t, u_t)
\end{equation}
where $x_t = (x_t^p,x_t^a) \in X$ is the robot state, containing its pose $x_t^p$ and other variables $x_t^a$ such as velocity and acceleration and $u_t \in U$ is the control input, selected from a \textit{finite} space of admissible controls. A trajectory of the robot, for $t\in \mathbb{N}\cup \{\infty\}$, is a sequence of states $x_{0:t} :=x_0x_1\ldots x_t$, where $x_k\in X$ is the state at time $k$.


The robot operates in an environment modeled by a semantic map $\mathcal{M} := \{l_1,\ldots,l_M\}$ consisting of $M$ landmarks. Each landmark $l_i := (l_i^p,l_i^c) \in \mathcal{M}$ is defined by its pose $l_i^p$ and class $l_i^c\in\mathcal{C}$, where $\mathcal{C}$ is a finite set of classes (e.g., table, chair, door, etc.). The robot does not know the true landmark poses but has access to a probability distribution $\mathcal{P}$ over the space of all possible maps. Such a distribution can be produced by a semantic SLAM algorithm~\citep{slam++,semantic_sfm} and typically consists of a Gaussian distribution over the landmark poses and a discrete distribution over the landmark classes. More precisely, we assume $\mathcal{P}$ is determined by parameters $(\bar{l}^p, \Sigma^p, \{\rho_i^c\}_{i=1}^M)$ such that $l^p \sim \mathcal{N}(\bar{l}^p,\Sigma^p)$ and $l_i^c$ is generated by the probability mass function $\rho_i^c$. In this work, we suppose that the class of each landmark is known and leave the case of uncertaint landmark classes for future work.

%

\subsection{Temporal logic specifications}
We use linear temporal logic (\ac{ltl}) to specify the robot's task in the environment. \ac{ltl} formulas \citep{Pnueli198145} can describe temporal ordering of events along the robot trajectories and are defined by the following grammar:
$\phi:=p \mid \neg \phi \mid \phi_1\lor \phi_2 \mid \bigcirc \phi
\mid \phi_1\mathcal{U} \phi_2$,
where $p\in \calAP$ is an atomic proposition, and $\bigcirc$ and
$\mathcal{U}$ are temporal modal operators for ``next'' and
``until''. Additional temporal logic operators are derived from basic
ones: $\lozenge \varphi:= \truev\; \mathcal{U}\varphi$ (eventually)
and $\square \varphi := \neg \lozenge \neg \varphi$ (always).
We assume that the robot's task is given by an \ac{ltl} co-safe formula \cite{kupferman2001model}, which allows checking its satisfaction using a finite-length robot trajectory.

The \ac{ltl} formula is specified over a finite set of atomic propositions that are defined over the robot state space $X$ and the environment map $\calM$. Examples of atomic propositions include:
\begin{equation}
\label{eq:ap_def}
\begin{aligned}
\alpha_i^p(r): \; &d(x^p,l^p_i) \leq r &&\text{ for } r \in \mathbb{R}, i \in \{1, \ldots, M\},\\
\alpha_i^c(Y): \; &l^c_i \in Y &&\text{ for } Y \subseteq \mathcal{C}, i \in \{1, \ldots, M\}.
\end{aligned}
\end{equation}
Proposition $\alpha_i^p(r)$ evaluates true when the robot is within $r$ units distance from landmark $i$, while proposition $\alpha_i^c(Y)$ evaluates true when the class of the $i$-th landmark is in the subset $Y$ of classes. In order to interpret an \ac{ltl} formula over the trajectories of the robot system, we use a labeling function to determine which atomic propositions hold true for the current robot pose.

\begin{definition}[Labeling function\footnote{When the map is fixed, our labeling function definition reduces to the commonly-used definition in robotic motion planning under temporal logic constraints \citep{ding2011mdp}.}]
Let $\calAP$ be a set of atomic propositions and $\mathsf{M}$ be the set of all possible maps. A labeling function $L:X \times \mathsf{M} \rightarrow 2^\calAP$  maps a given robot state $x \in X$ and map $\mathcal{M} \in \mathsf{M}$ to a set $L(x, \calM)$  of atomic propositions that evaluate true.
\end{definition}
For robot trajectory $x_{0:t}$ and map $\calM \in \mathsf{M}$, the \emph{label sequence of $x_{0:t}$} in $\calM$, denoted $L(x_{0:t}, \calM)$, is such that $L(x_{0:t}, \calM) = L(x_0, \calM)L(x_1, \calM)L(x_2, \calM)\ldots L(x_t, \calM)$.
Given an \ac{ltl} co-safe formula $\varphi$, one can construct a \ac{dfa} $\calA_\varphi=(Q,2^\calAP, T, q_0, F) $ where $Q, 2^\calAP, q_0, F$ are a finite set of states, the alphabet, the initial state, and a set of final states, respectively. $T: Q\times 2^\calAP\rightarrow Q$ is a transition function such that $T(q,a)$ is the state that is reached with input $a$ at state $q$. We extend the transition function in the usual way\footnote{Notation: Let $A$ be a finite set. Let $A^\ast, A^\omega$ be the set of finite and infinite words over $A$. Let $\lambda = A^0$ be the empty string. Abusing notation slightly, we use $\emptyset$ and $\lambda$ interchangeably. For $w \in A^\omega$, if there exist $u\in A^\ast$ and $v \in A^\omega $ such that $w=uv$ then $u$ is a \emph{prefix} of $w$ and $v$ is a \emph{suffix} of $w$.}:
 $T(q,uv)=T(T(q,u), v)$ for $u,v\in (2^\calAP)^\ast$. A word $w$ is \emph{accepted} in $\calA_\varphi $ if and only if $T(q_0,w)\in F$. The set of words accepted by $\calA_\varphi$ is the \emph{language} of $\calA_\varphi$, denoted $\mathcal{L}(\calA_\varphi)$.


We say that a robot trajectory $x_{0:\infty}$ satisfies the \ac{ltl} formula $\varphi$ in the map $\calM$ if and only if there is $k\ge 0$ such that $L(x_{0:k}, \calM) \in \mathcal{L}(\calA_{\varphi})$. Then, $x_{0:k}$ is called a \textit{good prefix} for the formula $\varphi$. Furthermore, $\calA_\varphi$ accepts exactly the set of good prefixes for $\varphi$ and for any state $q\in F$, it holds that $T(q,a )\in F$ for any $a\in 2^\calAP$.

We are finally ready for a formal problem statement.
\begin{problem}
\label{prob:gssp}
Given an initial robot state $x_0 \in X$, a semantic map distribution $\mathcal{P}$, and an \ac{ltl} co-safe formula $\varphi$ represented by a \ac{dfa} $\calA_\varphi$, choose a stopping time $\tau$ and a sequence of control policies $u_t \in U$ for $t = 0,1,\ldots,\tau$ that maximizes the probability of the robot satisfying $\varphi$ in the true environment $\mathcal{M}$ while minimizing its motion cost:
\begin{align*}
  \min_{\tau,u_0,u_1,\ldots, u_\tau} \; & \mathbb{E}\big[\sum_{t=0}^\tau c(x_t,x_{t+1})\big]  + \kappa \mathbb{P}( q_{\tau+1} \notin F)\\
  \text{s.t.} \quad & x_{t+1} = f(x_t,u_t),\\
  &q_{t+1} = T(q_t,L(x_{t+1},\mathcal{M})), \; \forall 0\le t <\tau, 
\end{align*}
where $c$ is a positive-definite motion cost function, which satisfies the triangle inequality, and $\kappa \geq 0$ determines the relative importance of satisfying the specification versus the total motion cost. 
\end{problem}

\textit{Remark}: The optimal cost of Problem \ref{prob:gssp} is bounded below by 0 due to the assumptions on $c,\kappa$ and bounded above by $\kappa$, obtained by stopping immediately without satisfying $\varphi$, i.e., $\tau = 0$.

\section{Planning to be Probably Correct}
\label{sec:reduction}
The map uncertainty in Problem \ref{prob:gssp} leads to uncertainty in
the evaluation of the atomic propositions and hence to uncertainty in
the robot trajectory labeling. In the meantime, the automaton
state cannot be observed.  Rather than solving the resulting
optimal control problem with partial observability, we
propose an alternative solution that generates a near-optimal
plan with a probabilistic correctness guarantee for the temporal logic
constraints. The main idea is to convert the original semantic map
distribution to a high-confidence deterministic representation and solve a
deterministic optimal control problem with this new
representation. The advantage is that we can solve the deterministic
problem very efficiently and still provide a correctness guarantee.
This avoids the need for sampling-based methods in the continuous space 
of map distributions, which become computationally expensive when planning in large environments.
To this end, we use a confidence region around the mean
$\bar \calM:=\{(\bar{l}_i^p,l_i^c)\}_{i=1}^M$ of the semantic map
distribution $\mathcal{P}$ to extend the definition of the labeling
function. 

\begin{definition}[$\delta$-Confident labeling function]
\label{def:confident_labeling}
Given a robot state $x \in X$, a map distribution $\mathcal{P}$, and a parameter $\delta\in (0,1)$, a $\delta$-confident labeling function is defined as follows:
\[
L^\delta(x,\mathcal{P}) := \begin{cases}
  L(x,\bar \calM) & \text{if  $L(x, \bar \calM) = L(x, m)$ for all maps}\\
  & \text{$m$ in the $\delta$-confidence region of $\mathcal{P}$,}\\
  \emptyset & \text{ otherwise}.
\end{cases}
\]
\end{definition}

We now explain the intuition for defining the $\delta$-confident labeling
function as in Def.~\ref{def:confident_labeling}. For a given robot
trajectory $x_{0:t}$, rather than maintaining a distribution over the
possible label sequences, the robot keeps only a sequence of labels
that, with probability $\delta$, is a subsequence\footnote{For a word
  $u\in A^\omega$, $u$ is a \emph{subsequence} of $w$ if $u$ can be obtained from $w$ by replacing
  symbols with the empty string $\lambda$.} of the label sequence
$L(x_{0:t}, \calM)$ in the true environment. This statement is made precise in the following proposition.

\begin{proposition}
\label{prop:relatinglabelseq}
Given a robot trajectory $x_{0:t}$ and a map distribution
$\mathcal{P}$, $L^\delta(x_{0:t}, \mathcal{P}) $ is a subsequence of
$L(x_{0:t}, \calM)$ with probability $\delta$.
\end{proposition}
\begin{proof}
See Appendix~\ref{subsec:proof_relatinglabelseq}
\end{proof}


Intuitively, a label $L(x_k,\calM)$ is preserved at the $k$-th position of $L^\delta(x_{0:t}, \mathcal{P})$ if for any two sample maps $m, m'$ in the $\delta$-confidence region, $L(x_i,m)= L(x_i,m')$. Otherwise, it is replaced with $\emptyset$.  Next, we show that
when the \ac{ltl} formula $\varphi$ satisfies a particular property,
if $L^\delta(x_{0,t},\mathcal{P})$ is accepted by the \ac{dfa}
$\calA_\varphi$, then with probability $\delta$, $x_{0:t}$ satisfies
the \ac{ltl} specification $\varphi$. 
The required property is that the formal language characterization of the logic formula translates to a \emph{simple polynomial} \citep{pin2010mathematical}. An $\omega$-regular language $\cal L$ over an alphabet $A$ is
\emph{simple monomial} if and only if it is of the form
\[
A^\ast a_1 A^\ast a_2A^\ast \ldots A^\ast a_k A^\ast 
(A^\ast b_1  A^\ast b_2A^\ast \ldots A^\ast b_\ell A^\ast)^\omega
\]
where $a_1,a_2,\ldots, a_k, b_1,b_2,\ldots, b_\ell \in A$, $k \ge 0$,
and $\ell \ge 0$. A finite union of simple monomials is called a
\emph{simple polynomial}.


\begin{theorem}
\label{thm1}
  If the the language $\mathcal{L}(\calA_\varphi)$ of $\calA_\varphi$ is a
  simple polynomial, then $L^\delta(x_{0:\tau},\mathcal{P}) \in \mathcal{L}(\calA_\varphi)$
  implies that $\mathbb{P}(L(x_{0:\tau},\mathcal{M}) \in \mathcal{L}(\calA_\varphi)) = \delta$.
\end{theorem}
\begin{proof}
See Appendix~\ref{subsec:proof_thm1}.
\end{proof}

The significance of Thm. \ref{thm1} is that it allows us to reduce the stochastic 
control problem with an uncertain map (Problem \ref{prob:gssp}) to a deterministic
shortest path problem. We introduce the following  product system to 
facilitate the conversion.

\begin{definition}[$\delta$-Probably correct product system]
\label{def:product}
  Given the robot system in \eqref{eq:system}, the map distribution
  $\mathcal{P}$, the automaton $\calA_\varphi$, and a parameter
  $\delta\in (0,1)$, a \emph{$\delta$-probably correct product system}
  is a tuple $\calG^{\delta} = \langle S,U, \Delta,s_0, S_F \rangle$ defined as follows.
\begin{itemize}
  \item $S=X \times Q$ is the  product state space.
  \item $\Delta: S \times U \rightarrow S$ is a transition function such that 
    $\Delta ((x,q),u) = (x',q')$ where $x'=f(x,u)$ and $q'=T(q, L^\delta(x',\mathcal{P}))$.
    It is assumed that $T(q,\emptyset)=q,\; \forall q\in Q$.
  \item $s_0=(x_0,q_0)$ is the initial state.
  \item $S_F = X \times F$ is the set of final states.
\end{itemize}
\end{definition}

For the subclass of \ac{ltl} co-safe formulas whose languages are
simple polynomials, Thm.~\ref{thm1} guarantees that the projection on $X$ of any trajectory 
$s_{0:t}$ of $\calG^\delta$ that reaches $S_F$ in the $\delta$-confident map 
has probability $\delta$ of satisfying the specification in the true map.
The implications are explored in Sec.~\ref{sec:deterministic}.


Before we proceed, however, it is important to know to what extent the expressiveness
of \ac{ltl} is limited by restricting it to the subclass of simple polynomials. 
In Appendix~\ref{subsec:logic_sp}, we show that such \ac{ltl} formulas can express 
reachability and sequencing properties. Moreover, with a slight modification of 
Def.~\ref{def:product}, we can also ensure the correctness of plans with respect to safety constraints.

Consider safety constraints in the following form
$\square \phi_{safe}$ with $\phi_{safe}$ being a propositional logic
formula over $\calAP$. For example, an obstacle avoidance requirement
is given by $\square d(x,x_o)\ge r $ where $x_o$ are the coordinates
of an obstacle.  When the \ac{ltl} formula includes such safety
constraints, we need to modify the transition function in
Def.~\ref{def:product} in the following way. For any state $s\in S$
and any input $u\in U$, let $s'= (x',q') =\Delta(s,u)$. Then, if there
exists at least one $m$ in the $\delta$-confidence region of $\cal P$
such that the propositional logic formula corresponding to $L(x',m)$
implies\footnote{Given a label $L(x',m) \subseteq \calAP$, the corresponding
  propositional logic formula is
  $\bigwedge_{\alpha_i\in L(x',m)} \alpha_i \land \bigwedge_{\alpha_j \in
    \calAP \setminus L(x',m) } \neg \alpha_j$.} $\neg \phi_{safe}$, let $\Delta(s,u) = \sink$, where
$\sink $ is a non-accepting sink state that satisfies
$\Delta(\sink,u)= \sink$ for any $u\in U$. Thus, the state $\sink$ will not be visited
by any trajectory of $\calG^\delta$ that reaches $S_F$, which means the safety constraint 
will be satisfied with probability $\delta$ in the true environment. The following toy example illustrates the concepts.

\begin{example}
  In Fig.~\ref{fig:ex}, a mobile robot is tasked with visiting at least one 
  landmark in an uncertain environment. Formally, the
  \ac{ltl} specification of the task is $\varphi:= \lozenge p$ where
  $p:= d(x,l_i^p)\le 1$ for any landmark $i$. Given the robot
  trajectory $x_{0:t}$ represented by the dashed line in the figure,
  when the robot traverses the $95\%$-confidence region of $l_1$'s pose distribution,
  it cannot confidently (with confidence level $\delta=0.95$) decide the value 
  of $p$ in the true map because for some map realizations, $d(x,l_1^p) > 1$. 
  On the other hand, when the robot is near $l_2$, $p$ evaluates true because
  a unit ball around the robot covers the entire $95\%$-confidence region of $l_2$'s 
  pose distribution. Let $\mathcal{B}_r(l)$ a ball centered at $l$ with radius $r$. 
  The label sequence of $x_{0:t}$ in the true environment is 
  $L(x_{0:t}, \calM) = \emptyset^{k_1} \{p\}^{k_2} \emptyset^{k_3}\{
  p\} = \{p\}^{k_2+1}$
  where $k_1,k_2,k_3$ are the numbers of steps before reaching
  $\mathcal{B}_1(l^p_1)$, in $\mathcal{B}_1(l^p_1)$, and after
  leaving $\mathcal{B}_1(l^p_1)$ but before reaching
  $\mathcal{B}_2(l^p_2)$. The label sequence
  $L^{0.95}(x_{0:t}, \mathcal{P}) = \emptyset^{k_1+k_2+k_3}\{p\} =
  \{p\}$.
  Clearly,
  $ L^{0.95}(x_{0:t}, \mathcal{P}) $ is a subsequence of $ L(x_{0:t}, \calM) $.
  Moreover, the trajectory satisfies the \ac{ltl} specification which is a
  reachability constraint.

\begin{figure}[ht]
\centering
\includegraphics[width=0.5\textwidth,trim=10mm 15mm 0mm 10mm, clip]{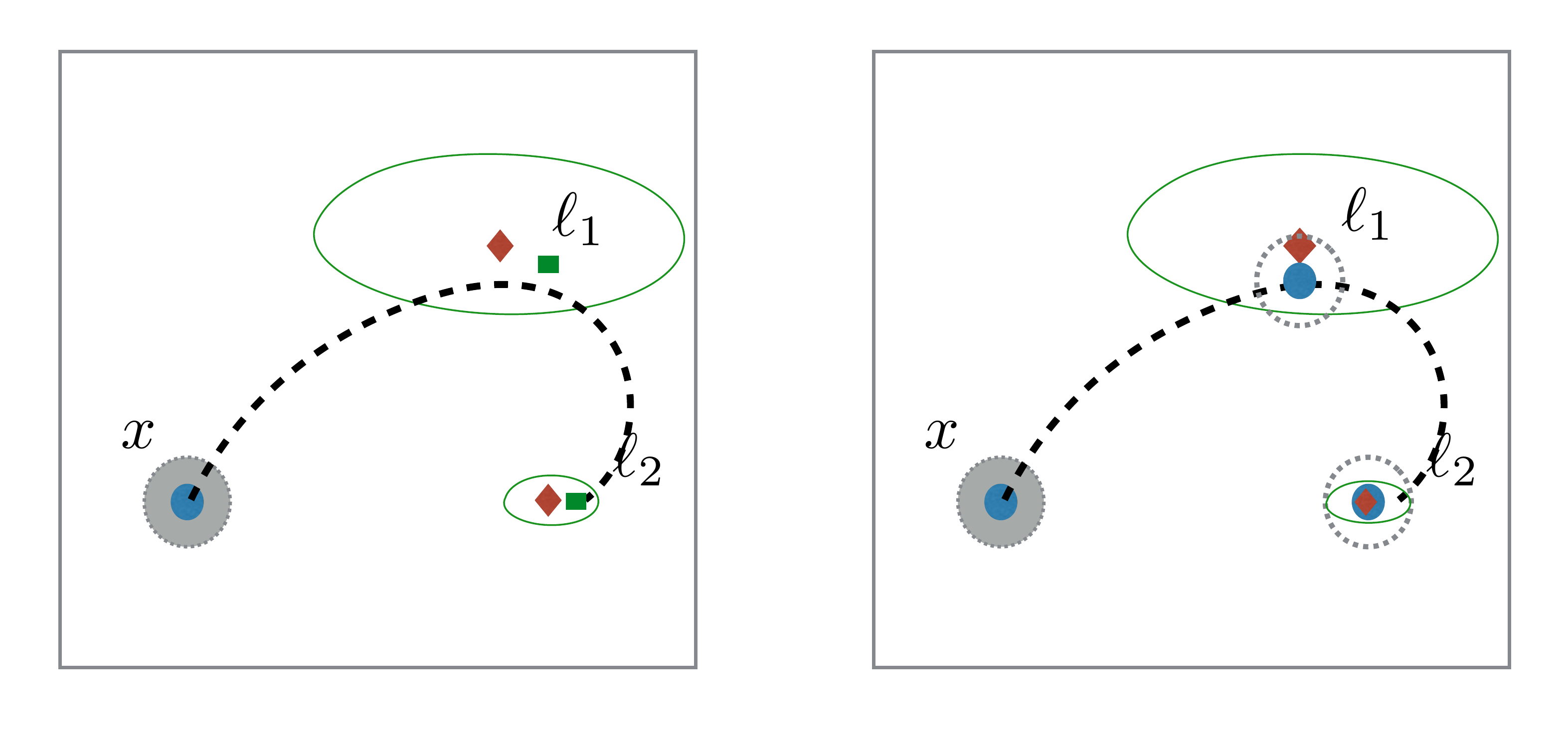}
\caption{Example: robot with a reachability objective. The red diamonds
  are the means of the landmarks' pose distributions and the ellipsoids are the
  associated $95\%$-confidence regions. The robot is
  represented by the blue dot. The green squares represent the true
  locations of landmarks $l_1$ and $l_2$. \label{fig:ex}}
\end{figure}

In the case of a safety constraint, e.g.,
$\varphi_{safe}= \square  p$ where $p:= d(x, l_i^p)> 1$ for any
landmark $i$, once the robot gets close to $l_1$ the $\delta$-probably
correct
product system will transition to the non-accepting sink because there 
exists a sample map $m$ such that the safety constraint is violated.
Thus, in any run that is safe in $\calG^\delta$, the robot is able to 
safely avoid both $l_1$ and $l_2$ with probability $\delta$.
\end{example}

\section{Reduction to Deterministic Shortest Path}
\label{sec:deterministic}
For a fixed confidence $\delta$, due to Thm. \ref{thm1}, we can
convert Problem \ref{prob:gssp} to a deterministic shortest path
problem within the probably correct product system $\calG^\delta$. In this case,
$\mathbb{P}(q_{\tau+1} \notin F) \in \{0,1\}$ and the optimal solution
to Problem \ref{prob:gssp} is either to stop immediately ($\tau =0$),
incurring cost $\kappa$, or to find a robot trajectory satisfying
$\varphi$ with cumulative motion cost less than $\kappa\delta$. The
latter corresponds to the following deterministic problem.

\begin{problem}
\label{prob:dsp}
Given an initial robot state $x_0 \in X$, a semantic map distribution
$\mathcal{P}$, a confidence $\delta \in (0,1)$, and an \ac{ltl}
co-safe formula $\varphi$ represented by a \ac{dfa} $\calA_\varphi$,
choose a stopping time $\tau$ and a sequence of control inputs
$u_t \in U$ for $t = 0,1,\ldots$ that minimize the motion cost of
a trajectory that satisfies $\varphi$:
\begin{align*}
  \min_{u_0,u_1,\ldots, u_\tau} \; &\sum_{t=0}^\tau c(x_t,x_{t+1})\\
  \text{s.t.} \quad & x_{t+1} = f(x_t,u_t),\\
  &q_{t+1} = T(q_t,L^\delta(x_t,\mathcal{P})),\; \forall 0\le t < \tau,\\
  &q_{\tau+1} \in F,\\
  &\sum_{t=0}^\tau c(x_t,x_{t+1}) \leq \kappa\delta.
\end{align*}
\end{problem}
If Problem \ref{prob:dsp} is infeasible, it is best in
Problem \ref{prob:gssp} to stop immediately ($\tau=0$), incurring cost $\kappa$; 
otherwise, the robot should follow the
control sequence $u^*_{0:\tau}$ computed above and the corresponding
trajectory $x^*_{0:\tau+1}$ to incur cost:
\[
\sum_{t=0}^\tau c(x_t^*,x^*_{t+1}) + \kappa (1 - \delta) \leq \kappa
\]
in the original Problem \ref{prob:gssp}. Since Problem \ref{prob:dsp}
is a deterministic shortest path problem, we can use any of the
traditional motion planning algorithms, such as RRT~\citep{rrt},
RRT*~\citep{rrt_star,rrg_mu_calc} or A*~\citep{ara_star} to solve
it. We choose A* due to its completeness guarantees\footnote{To guarantee completeness of A* for Problem \ref{prob:dsp}, the robot state space $X$ needs to be assumed bounded and compact and needs to discretized.}~\citep{astar} and
because the automaton $\calA_\varphi$ can be used to guide the search
as we show next.



\subsection{Admissible Heuristic}
The efficiency of A* can be increased dramatically by designing an
appropriate heuristic function to guide the search. Given a state
$s := (x,q)$ in the product system (Def. \ref{def:product}), a
heuristic function $h:S\rightarrow \mathbb{R}$ provides an estimate of the optimal cost $h^*(s)$ from $s$ to the goal set
$S_F$. If the heuristic function is \emph{admissible}, i.e., never
overestimates the cost-to-go ($h(s) \leq h^*(s), \; \forall s \in S$),
then A* is optimal~\citep{astar}.

\citet{LPH15} propose a distance metric to evaluate the progression of
an automaton state with respect to an \ac{ltl} co-safe formula. We use
a similar idea to design an admissible heuristic function. We
partition the state space $Q$ of $\calA_\varphi$ into level sets as
follows. Let $Q_0 := F$ and for $i \geq 0$ construct
$Q_{i+1} := \{q\in Q \setminus \bigcup_{k=0}^i Q_k \mid \exists q'\in
Q_i, a\in 2^\calAP, \text{ such that } T(q,a)=q' \}$.
The generation of level sets stops when $Q_i = \emptyset$ for some
$i$. Further, we denote the set of all sink states by
$Q_\infty$. Thus, given $q \in Q$ one can find a unique level set
$Q_i$ such that $q \in Q_i$. We say that $i$ is the level of $q$ and
denote it by $\level(q)=i$.


\begin{proposition}
  \label{prop:levels}
  Let $s_{0:t}$ be a trajectory of the product system $\calG^\delta$ that
  reaches $S_F$, i.e., $s_t\in S_F$. Then, for any $0\le k < t$, given
  $s_k=(x_k, q_k)$ and $s_{k+1}=(x_{k+1},q_{k+1})$, it holds that
  $\level(q_k)\le \level(q_{k+1})+1$.
\end{proposition}
\begin{proof}Since $T(q_k, L^\delta(x_{k+1},\mathcal{P}))=q_{k+1}$, if $q_{k+1} \in Q_i$ for some level $i$, then, by construction of the level sets, either $q_k\in Q_{i+1}$ or $q_k \in \bigcup_{j= 0}^i Q_j$.
\end{proof}

By construction of the level sets, the automaton states $q_{0:t}$,
associated with any trajectory $s_{0:t}$ of the product system that
reaches a goal state ($s_t \in S_F$), have to pass through the level
sets sequentially. In other words, if $\level(q_0) = i$, then there
exists a subsequence $q'_{0:i}$ of $q_{0:t}$ such that
$\level(q'_1) = i-1$, $\level(q'_2) = i-2,\ldots,\level(q'_i) =
0$.
Thus, we can construct a heuristic function that underestimates the
cost-to-go from some state $s:=(x,q) \in S$ with $\level(q)=i$ by
computing the minimum cost to to reach a state $s':=(x',q')$ such that
$\level(q') \in\{ i-1,i\}$ and $q\ne q'$. To do so, we determine all
the labels that trigger a transition from $q$ to $q'$ in
$\calA_\varphi$ and then find all the robot states $\mathcal{B}$ that
produce those labels. Then, $h(x,q)$ is the minimum distance from $x$
to the set $\mathcal{B}$. The details of this construction and other
functions needed for A* search with \ac{ltl} specifications, are
summarized in Alg. \ref{alg:astar}.

\begin{figure}[hb]
\centering
\includegraphics[width=0.5\textwidth]{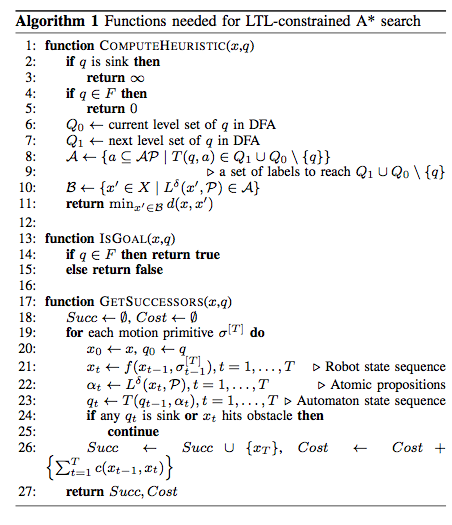}
\caption{Functions needed for LTL-constrained A* search}
\label{alg:astar}
\end{figure}

\begin{proposition}
\label{prop:admissible_h}
The heuristic function in Alg. \ref{alg:astar} is admissible.
\end{proposition}
\begin{proof}
See Appendix~\ref{subsec:proof_admissible_h}.
\end{proof}

Prop. \ref{prop:admissible_h} guarantees that A* will either find the optimal solution to Problem \ref{prob:dsp} or will report that Problem \ref{prob:dsp} is infeasible. In the latter case, the robot cannot satisfy the logic specification with confidence $\delta$ and it should either reduce $\delta$ or stop planning.

Note that while the heuristic function is admissible, it is not guaranteed that it is also consistent. Consider two arbitrary states $(x,q)$ and $(x',q')$ with $\level(q) = n$ and $\level(q')=n+1$. It is possible that the cost to get from $(x,q)$ to a place in the environment, where a transition to level $n-1$ occurs, is very large, i.e., $h(x,q)$ is large, but it might be very cheap to get from $(x',q')$ to $(x,q)$ and vice versa. In other words, it is possible that the following inequalities hold:
\[
c(x,x') + h(x',q') \leq c(x,x') + c(x',x) < h(x,q),
\]
which makes the heuristic in Alg. \ref{alg:astar} inconsistent. We emphasize that, even with an inconsistent heuristic, $A^*$ can be very efficient if a technique such as bi-directional pathmax is employed to propagate heuristics between neighboring states~\citep{astar_incos}.

\subsection{Accelerating A* Search using Motion Primitives} 
\begin{figure}[t]
  \centering
  \includegraphics[width=0.8\linewidth]{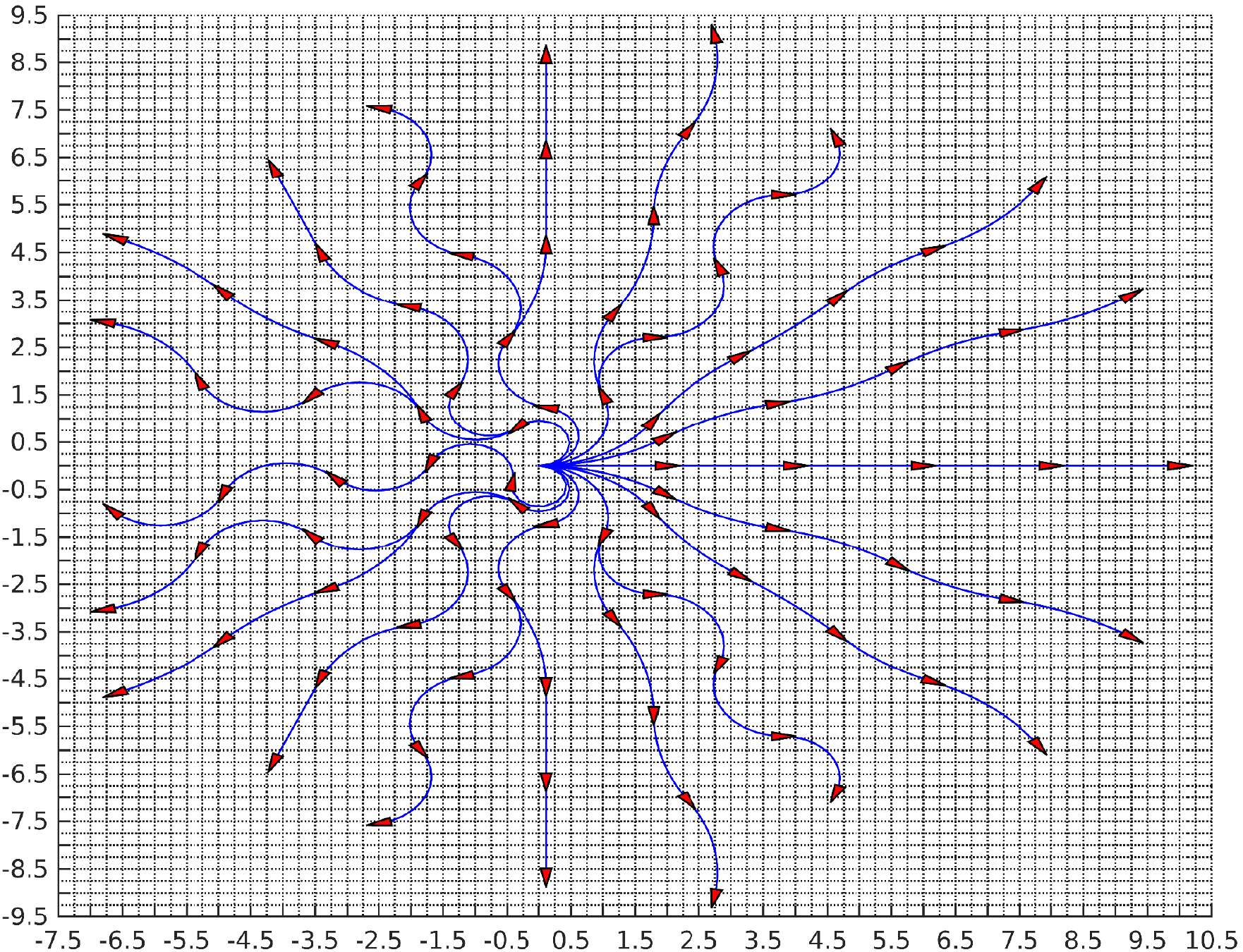}
  \caption{The set of motion primitives used for LTL-constrained A* planning. Each segment contains 5 robot poses indicated by the red triangles. Each robot pose represents a different motion primitive.}
  \label{fig:motion_primitives}
\end{figure}

If willing to sacrifice completeness, one can significantly accelerate the A* search by using motion primitives for the robot in order to construct a lattice-based graph \citep{mprim}. The advantage of such a construction is that the underlying graph is sparse and composed of dynamically-feasible robot trajectories that can incorporate a variety of constraints. We present motion primitives for a differential-drive robot in Fig. \ref{fig:motion_primitives} but much more general models can be handled \citep{mprim_humanoid}.

A motion primitive is similar to the notion of macro-action \cite{macro_actions71,macro_actions05} and consists of a collection of control inputs $\sigma^{[T]}:=(u_0,u_2,\ldots,u_{T-1})$ that are applied sequentially to a robot state $x_t$ so that:
\[
x_{t+1+k} = f(x_{t+k}, \sigma^{[T]}(k)), \; k = 0,\ldots,T-1.
\]
Instead of using the original control set $U$, we can plan with a set $\bar{U}:=\{\sigma^{[T_j]}_j\}$ of motion primitives (see $\textproc{\footnotesize GetSuccessors}$ in Alg. \ref{alg:astar}). In our experiments, the motion primitives were designed offline. Twenty
locations with outward facing orientations were chosen on the
perimeter of a circle of radius 10 m. A differential-drive controller
was used to to generate a control sequence of length $5$ that would
lead a robot at the origin to each of the selected
locations. Fig. \ref{fig:motion_primitives} shows the resulting set of
motion primitives. They are wavy because the controller tries to
follow a straight line using a discrete set of velocity and
angular-velocity inputs.

\noindent{\textbf{Summary.}}
We formulate the temporal logic planning
with a semantic map distribution as a stochastic optimal control
problem (Problem~\ref{prob:gssp}). Since Problem~\ref{prob:gssp} is
intractable, we reduce it to a deterministic shortest path problem (Problem~\ref{prob:dsp})
with probabilistic correctness guaranteed by Thm~\ref{thm1}. We can solve Problem~\ref{prob:dsp} optimally using A* because the heuristic function proposed in Alg. \ref{alg:astar} is admissible (Prop. \ref{prop:admissible_h}). The obtained solution is partial
with respect to Problem~\ref{prob:gssp} because, rather than a controller that trades off the probability of satisfying the specification and the total motion cost, it provides the optimal controller in a subspace of
deterministic controllers that guarantee that the probability of satisfying the specification is $\delta$.

\section{Examples}
\label{sec:application}

In this section, we demonstrate the method for \ac{ltl}-constrained
motion planning with a differential-drive robot with state
$\mathsf{x}_t := (x_t,y_t,\theta_t)^T \in SE(2)$, where $(x_t,y_t)$
and $\theta_t$ are the 2D position and orientation of the robot,
respectively. The kinematics of the robot are discretized using a
sampling period $\tau$ as follows:
\begin{equation*}
\begin{pmatrix}
x_{t+1}\\
y_{t+1}\\
\theta_{t+1}
\end{pmatrix}\!=\!
\begin{pmatrix}
x_{t}\\
y_{t}\\
\theta_{t}
\end{pmatrix} +
\begin{cases}
\begin{pmatrix}
\tau \nu \cos(\theta_t+\tau \omega/2)\\
\tau \nu \sin(\theta_t +\tau \omega/2)\\
\tau \omega
\end{pmatrix},\; \abs{\tau\omega}<0.001,\\
\begin{pmatrix}
 \frac{\nu}{\omega}( \sin(\theta_t+\tau \omega)-\sin \theta_t)\\
\frac{\nu}{\omega}( \cos \theta_t) -\cos(\theta_t+\tau \omega))\\
\tau \omega
\end{pmatrix},\; \textrm{else.}
\end{cases}
\end{equation*}
The mobile robot is controlled by motion primitives in
Fig. \ref{fig:motion_primitives}, whose segments are specified by
$\nu = 1\;m/\textrm{s}$, $\tau = 2\;\textrm{s}$, and
$\omega \in [-3,3]\;\textrm{rad}/\textrm{s}$.

The \ac{ltl} constraints were specified over the two types of atomic
propositions for object classes
$\mathcal{C} = \{\sq, \cir,\hex, \dia, \tri\}$. Proposition
$\alpha_i^c(h)$ means the class of $i$-th landmark is $h$ for
$h \in \mathcal{C}$. Proposition
$\alpha_i^p(r) :\; d(x^p -l_i^p) \le r$ means the robot is $r$-close
to landmark $l_i$.

The following \ac{ltl} specification was given to
the robot:
\[
\varphi :\;( \lozenge (\phi_1 \land \lozenge(\phi_2\land \lozenge \phi_3))
  \land \lozenge \phi_4 ) \land \square \phi_{safe}
\]
where $\phi_i$, $i=1,\ldots, 4$ are the following propositional logic formulas:
\begin{align*}
\phi_1:&=\alpha_i^p(1) \land \alpha_i^c(\tri),  &i& \in \{1,\ldots, M\}\\
\phi_2:&=\alpha_i^p(2) \land \alpha_i^c(\dia) \land \alpha_j^p(2) \land \alpha_j^c(\cir), &i,j&\in \{1,\ldots, M\}\\
\phi_3:&=\alpha_i^p(2) \land \alpha_i^c(\sq) \land \alpha_j^p(2) \land \alpha_j^c(\cir),  &i,j&\in \{1,\ldots, M\}\\
\phi_4:&=\alpha_i^p(1) \land \alpha_i^c(\hex),  &i& \in \{1,\ldots, M\}
\end{align*}
and for $i,j\in \{1,\ldots, M\}$ the safety constraint is:
\begin{flalign*}
&\phi_{safe}:= \square \neg (\alpha_i^p(2) \land \alpha_i^c(\sq) \land \alpha_j^p(2)\land \alpha_j^c(\sq)).&
\end{flalign*}
In other words, the robot needs to first visit a triangle, then go to a region where it is close to both a circle and a diamond, and finally visit a region where it is close to both a circle and a square, while visiting a hexagon at some point and avoiding getting stuck between any two squares.

Several case studies were carried out using a simulated semantic map distribution.
Robot trajectories with least cost that satisfy the \ac{ltl}
specification with confidence $\delta = 0.95$ for different initial
conditions were computed with A* and are shown in
Fig.~\ref{fig:traj}. Optimal paths with the same initial conditions
but different confidence parameters $\delta=0.95$ and $0.5$ are shown
in Fig.~\ref{fig:traj_diffconf}. As expected, we observe a trade-off
between the probability of satisfying the \ac{ltl} formula and the
total cost of the path. With a lower confidence ($\delta= 0.5$), the
total cost for satisfying the \ac{ltl} formula is also lower than that
of a path which satisfies the formula with a high confidence
($\delta=0.95$). Particularly, the uncertainty in the pose of triangle
$l_1$ with pose distribution $\bar l_1^p=(2.5,1.19)$ is the main reason
for the difference in the planned trajectories. With $\delta =0.95$,
even though the robot can reach the vicinity of triangle $l_1$, it
does not have enough confidence to ensure that the triangle would be
visited. Instead, it plans to visit another triangle $l_2$ for which
the uncertainty in the pose distribution $\bar l_2^p=(-5.18,12.04)$ is
smaller. Reducing the confidence requirements allows the robot to plan
a path that visits triangle $l_1$ and has a lower total cost compared
to that of visiting triangle $\l_2$.

\section{Conclusion}
\label{sec:conclusion}
This paper proposes an approach for planning optimal robot
trajectories that probabilistically satisfy temporal logic
specifications in uncertain semantic environments. By introducing a
$\delta$-confident labeling function, we show that the original stochastic optimal control problem in the continuous space of semantic map distributions can be reduced to a deterministic optimal control problem in the $\delta$-confidence region of the map distribution. Guided by the automaton representation of the \ac{ltl} co-safe specification, we develop an admissible A* algorithm to solve the deterministic problem. The advantage of our approach is that the deterministic problem can be solved very efficiently and yet the planned robot trajectory is guaranteed to have minimum cost and to satisfy the logic specification with probability $\delta$.

This work takes an initial step towards integration of semantic \ac{slam} and motion planning under temporal logic constraints. In future work, we plan to extend this method to handle the following:
\begin{inparaenum}[1)]
  \item landmark class uncertainty,
  \item robot motion uncertainty,
  \item a more general class of \ac{ltl} specifications,
  \item map distributions that are changing online.
\end{inparaenum}
Our goal is to develop a coherent approach for planning autonomous robot behaviors that accomplish high-level temporal logic tasks in uncertain semantic environments.

\begin{figure}[t]
  \centering
\begin{subfigure}[b]{0.5\textwidth}
  \includegraphics[width=\textwidth]{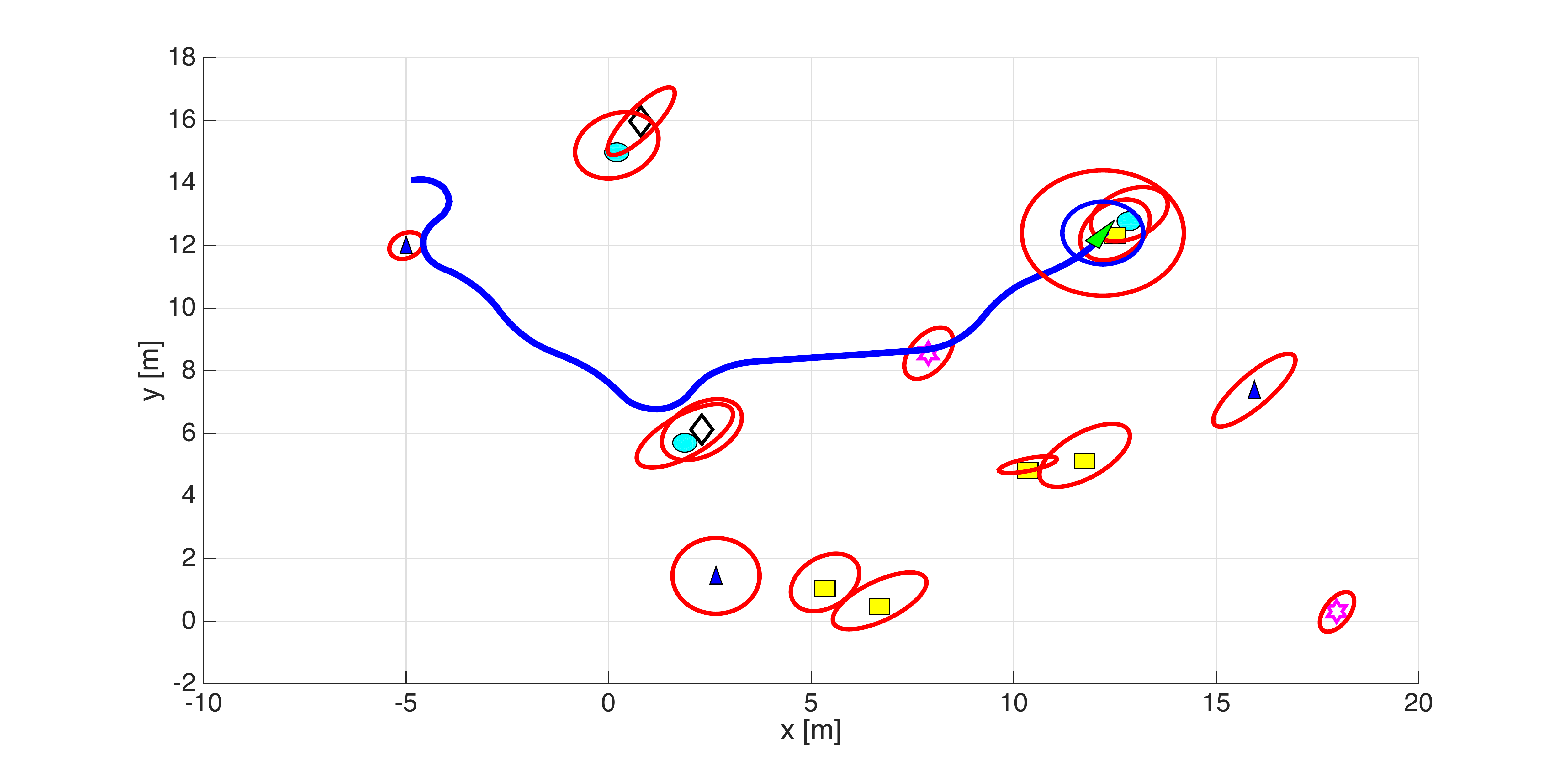}
\caption{$x_0=[-5,14,\pi/6]^T$, total motion cost $24\;m$.}
\end{subfigure}
\begin{subfigure}[b]{0.5\textwidth}
\includegraphics[width=\textwidth]{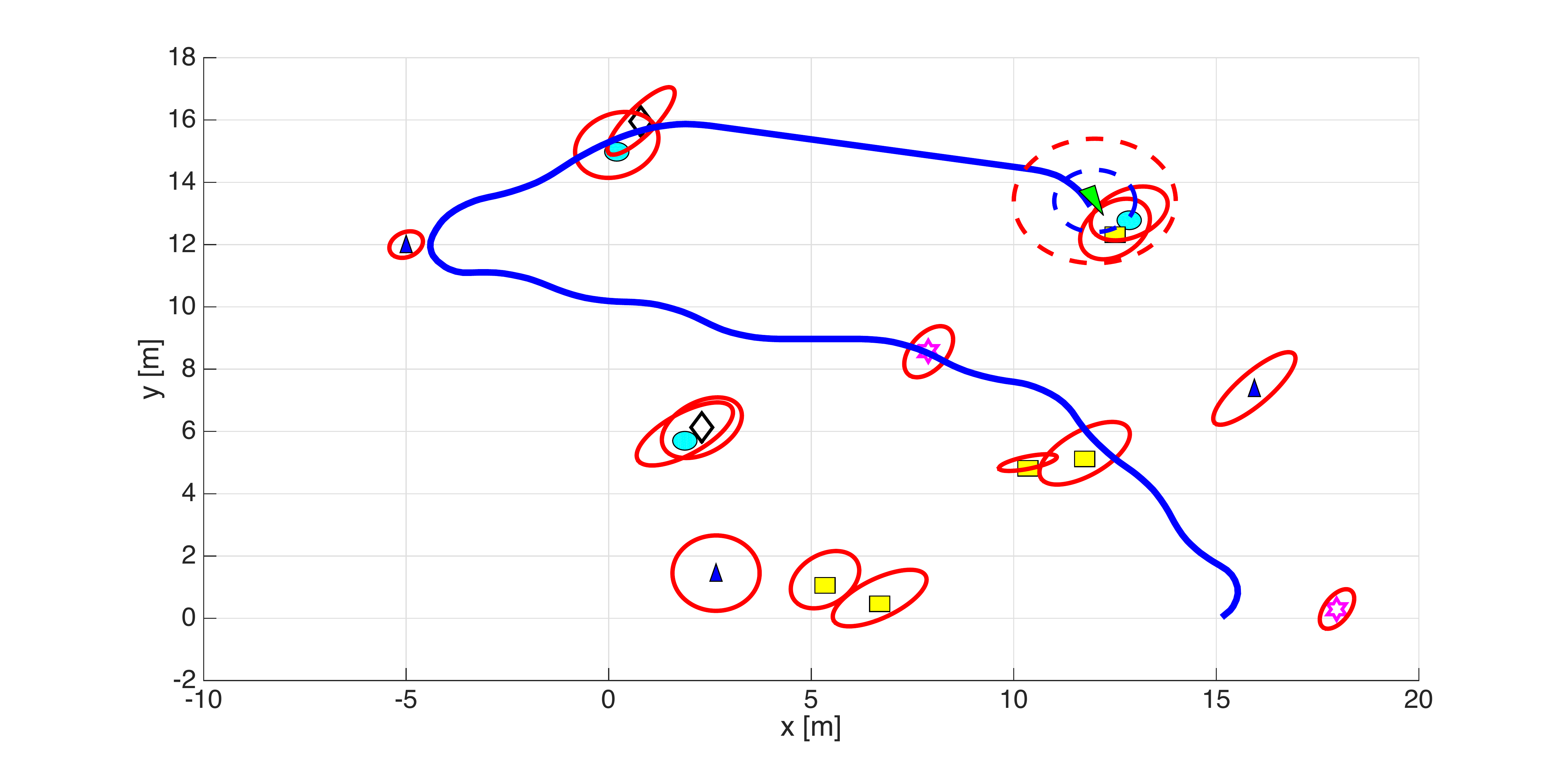}
\caption{$x_0=[15,0,\pi/6]^T$, total motion cost $44\;m$.}
\end{subfigure}
\caption{Robot trajectories obtained by the planning algorithm
  with different initial conditions. The green triangle represents the
  mobile robot. The blue dash circle indicates the
  unit ball and the red dash circle indicates a ball of radius $2\;m$.
  The red sold ellipses represent the $0.95$-confidence region of
  the landmark pose distributions. The landmark coordinates are drawn
  according to the mean of the pose distribution. \label{fig:traj}}
\end{figure} 

\begin{figure}[t]
\centering
\begin{subfigure}[b]{0.5\textwidth}
\centering
\includegraphics[width=\textwidth]{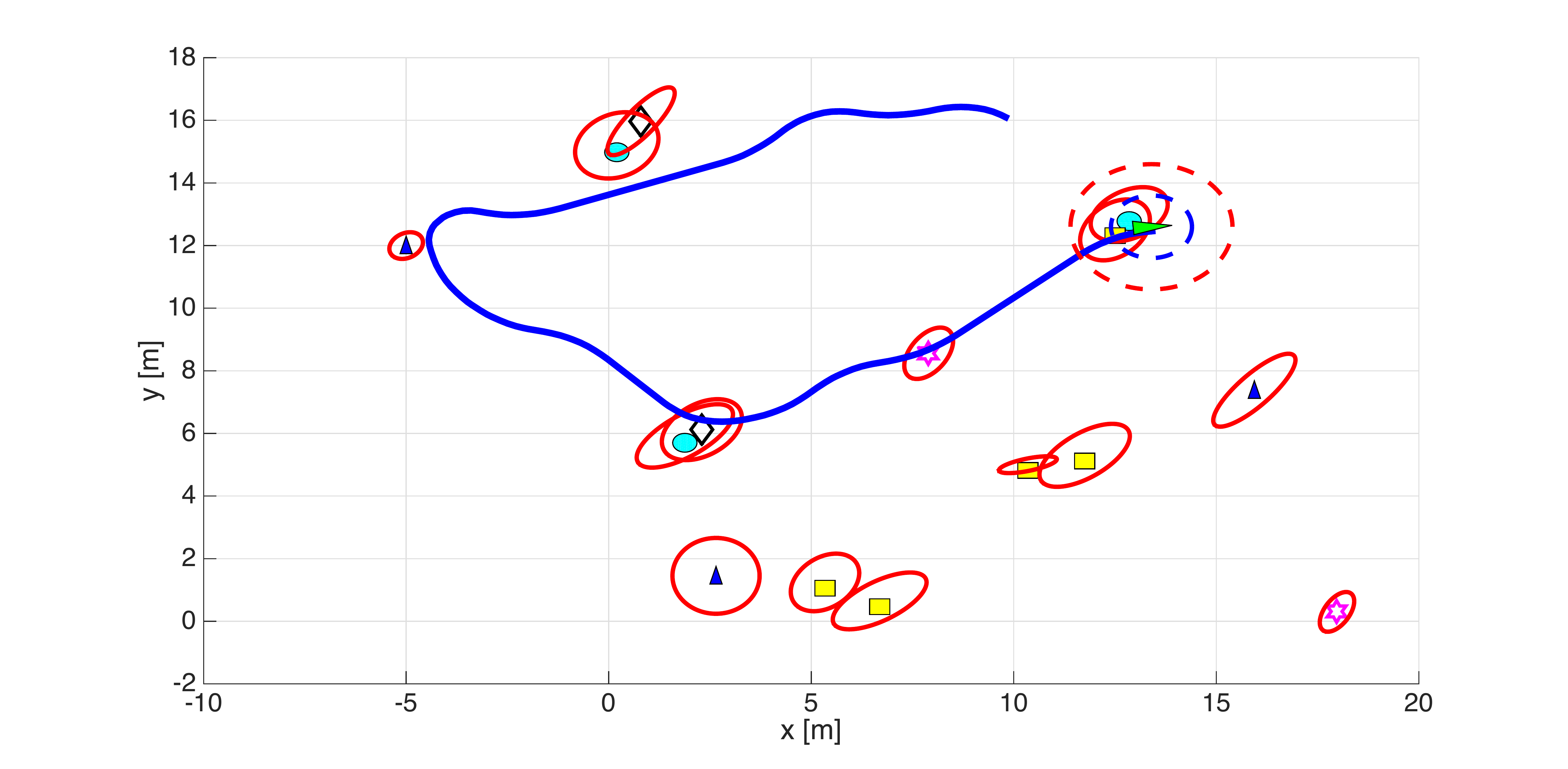}
\caption{The optimal plan with $0.95$-confidence and motion cost $38\;m$.}
\end{subfigure}
\vspace{2ex}
\begin{subfigure}[b]{0.5\textwidth}
\centering
  \includegraphics[width=\textwidth]{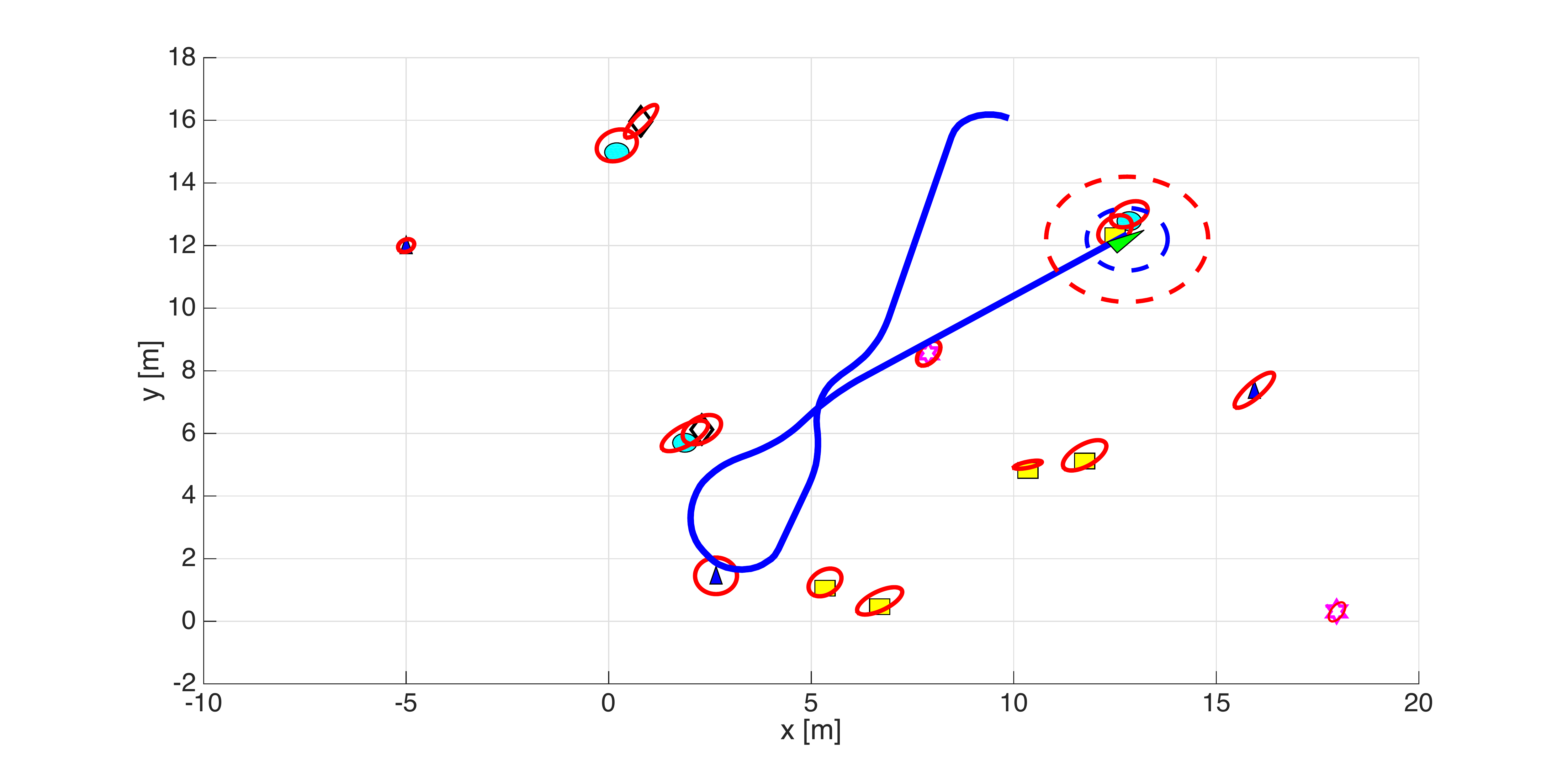}
\caption{The optimal plan with $0.5$-confidence and motion cost $34\;m$.}
\end{subfigure}
\caption{Robot trajectories obtained by the planning algorithm
  with different confidence parameters. The initial state is $x_0=[10,16,0.77\pi]^T$. \label{fig:traj_diffconf} }
\end{figure}


\appendix
\section{Appendix}
\label{sec:appendix}
\subsection{Proof of Prop. \ref{prop:relatinglabelseq}}
\label{subsec:proof_relatinglabelseq}
  Let $x_k$ be the state at time $k$. If for all samples $m$ in the $\delta$-confidence
  region of $\mathcal{P}$, $L(x_k,m)= L(x_k, \bar \calM)$, then $L^\delta(x_k, \mathcal{P}) = L(x_k,\bar \calM)$,
  which equals $L(x_k, \calM)$ with probability $\delta$. Otherwise, $L^\delta(x_k, \mathcal{P}) = \lambda$ (empty string).
  Since $k$ is arbitrary, we can conclude that $L^\delta(x_{0:t},
  \mathcal{P}) $ is a subsequence of $L(x_{0:t}, \calM)$ with
  probability $\delta$. \qed

\subsection{Proof of Thm. \ref{thm1}}
\label{subsec:proof_thm1}
Since the language $\mathcal{L}(\calA_\varphi)$ is a simple
polynomial, the following \emph{upward closure}
\cite{pin2010mathematical} property holds: For any word
$uv\in \mathcal{L}(\calA_\varphi)$ and any $a\in 2^\calAP$, it holds
that $uav \in \mathcal{L}(\calA_\varphi)$. In other words, if any empty string 
in a word from a simple polynomial language is replaced by a symbol in the alphabet, 
then the resulting word is still in the language.


For a given robot trajectory $x_{0:\tau}$, let
$L^\delta(x_{0:\tau}, \mathcal{P}) =b_0b_1\ldots b_\tau\in
\mathcal{L}(\calA_\varphi)$
be the $\delta$-confident label sequence and
$L(x_{0:\tau},\calM) = a_0a_1\ldots a_\tau$ be the true label
sequence. According to Prop.~\ref{prop:relatinglabelseq}, for each
$0\le i \le \tau$, either $a_i=b_i$ or $a_i\ne b_i$ and $b_i=\emptyset$. 
Thus, if $b_0b_1\ldots b_\tau$ belongs to
$\mathcal{L}(\calA_\varphi)$, $a_0a_1\ldots a_\tau$ must be in
$\mathcal{L}(\calA_\varphi)$ because it is obtained by replacing each empty string 
in $b_0b_1\ldots b_\tau$ with some symbol in $2^\calAP$ 
and the language $\mathcal{L}(\calA_\varphi) $ is
upward closed.\qed


\subsection{Characterization of \ac{ltl} co-safe formulas that
  translate to simple polynomials}
\label{subsec:logic_sp}
Formally, the subset of \ac{ltl} formulas is
defined by the grammar
\begin{equation}
\label{eq:grammar}\varphi:= \varphi_{reach} \mid \varphi_{seq} \mid \varphi \land
\varphi \mid \varphi \lor \varphi,
\end{equation}
where
$ \varphi_{reach} = \lozenge \phi \mid \varphi_{reach} \land
\varphi_{reach} \mid \varphi_{reach}\lor \varphi_{reach}$
represents the reachability,
$\varphi_{seq}: = \phi \mid \lozenge \varphi_{seq} \mid \lozenge (
\varphi_{seq} \land \lozenge \varphi_{seq})$
is a set of formulas describing sequencing constraints.  Here, $\phi$
is a propositional logic formula. 

\subsection{Proof of Prop. \ref{prop:admissible_h}}
\label{subsec:proof_admissible_h}
We proceed by induction on the levels in $\calA_\varphi$. In the base
case, $q \in Q_0 \equiv F$ and $h(x,q) = 0$ for any $x \in X$.
Suppose that the proposition is true for level $n$ and let $(x,q)$ be some state with $x \in X$ and $\level(q)=n+1$. As before, let $h^*(x,q)$ be the optimal cost-to-go. Due to Prop. \ref{prop:levels}, there are only three possibilities for the next state $(x',q')$ along the optimal path starting from $(x,q)$:
\begin{itemize}
  \item $\level(q') = n$: By construction of $h$:
    \[
     h^*(x,q) = c(x,x') + h^*(x',q') \geq h(x,q) + 0.
    \]
  \item $\level(q') = n+1$: Same conclusion as above.
  \item $\level(q') = k > n+1$: In this case, there exists another state $(x'',q'')$ later along the optimal path such that $\level(q'')=n+1$ (otherwise the optimal path cannot reach the goal set). Then, by the triangle inequality for the motion cost $c$:
    \begin{align*}
      h^*(x,q)&= c(x,x') + c(x',x'') + h^*(x'',q'')\\
              &\geq c(x,x'') + 0 \geq h(x,q)
    \end{align*}
\end{itemize}
Thus, we conclude that $h^*(x,q) \geq h(x,q) \geq 0$ for all $x \in X$ and $q \in Q$. \qed

{ \small
\bibliographystyle{abbrvplainnat}
\bibliography{ref.bib}
}

\end{document}